\documentclass[twocolumn]{article}

\usepackage{graphicx}
\usepackage{amssymb}
\usepackage{amsmath}
\usepackage[leftmargin=1em]{quoting}
\usepackage{caption}
\usepackage{multirow}

\newtheorem{pro}{Proposition} 
\newtheorem{lemma}{Lemma}
\newtheorem{proof}{Proof}

\newcommand{\C}{\mathcal{C}} 
\newcommand{\N}{\mathcal{N}} 
\newcommand{\Ob}{\mathcal{O}} 
\newcommand{\W}{\mathcal{W}} 
\newcommand{\R}{\mathbb{R}}
\newcommand{\A}{\mathcal{A}} 
\newcommand{\B}{\mathcal{B}} 

\newcommand{\x}{{{\bf x}}}
\newcommand{\g}{{\bf g}}
\newcommand{\rr}{{\bf r}}
\newcommand{\q}{{\bf q}}
\newcommand{\vv}{{\bf v}}

\date{\today}
\title{Probability Navigation Function for Stochastic Static Environments}

\author{Hacohen Shlomi, Shraga Shoval and Nir Shvalb}
\begin{document}
\maketitle
\begin{abstract}
Navigation function (NF) is widely used for motion planning; such a function is bounded, analytic, and guarantees convergence due to its Morse nature, while having a single minimum point at the target. This results in a safe path to the target. Originally, NF was developed for deterministic scenarios where the positions of the robot and the obstacles are known. Here we extend the concept of NF for static stochastic scenarios. \\
We assume that the robot, the obstacles and the workspace geometries are known discs, while their positions are random variables.  We define a \emph{Probability NF} (PNF) by introducing an additional \emph{permitted collision probability}, which limits the risks (to a set value) during robot motion.\\
We apply the Minkowski sum for the continuous case when considering the geometries with the Probability Density Functions (PDF). The PDF for collision is therefore the normalized convolution of the robot geometry, the obstacles geometries and their locations' PDFs.\\
We give an approximation for the permitted probability for collision. We then formulate an explicit function and prove that it is indeed a PNF. Finally, we exemplify our algorithm performances, and compare its results with a conventional NF algorithm.
\end{abstract}
\section{Introduction}
Motion planning for mobile robots has been extensively studied over the last three decades. Ideally one can assume that properties describing the robot movement, the environment and the obstacles are perfectly known. However, these parameters are often affected by substantial random factors (referred to as \emph{random variables}) due to measurement noises and physical process. {In the presence of uncertainties, even simple notions become non-conclusive, see for example the prediction of collisions between moving objects considered in \cite{kyriakopoulos1992distance}}.  
Researchers have studied algorithms to deal with process randomality: Lazanas and Latombe \cite{lazanas1995landmark} use \emph{a back-projection} algorithm to define areas where sensing may be considered accurate and these areas are added together to form "safe zones." A similar line of action is to maximize certainty by approaching known landmarks. A modification of \emph{A*} algorithm, introduced in \cite{lambert2003safe}, is constituted by adding a fourth dimension to the geometry. This additional dimension corresponds to the uncertainties, forming a  mathematical structure named \emph{"towers of uncertainties,"} which results in a path with a lower uncertainty. Pepy and Lambert \cite{pepy2006safe} define a configuration with the additional uncertainty data $(\sigma_x,\sigma_Y,\sigma_Z,\theta,\phi)$, and then introduce a \emph{safe-RRT} algorithm for a solution. Their experimental results show paths which indeed form safe trees that follow the walls in order to reduce the uncertainties. \\
A different viewpoint which guarantees convergence is to follow the \emph{robust control approach}. This can be applied for path planning in the sensor's image space \cite{mezouar2002path}, or be used to stabilize uncertain non-linear systems along nominal paths \cite{toussaint2000robust}, \cite{majumdar2013robust}. However, note that these methods require well established model equations.\\
A third approach is to replace the obstacles' locations with the computed probability for collision, for example as done for moving obstacles by Fulgenzi \cite{fulgenzi2007dynamic}.\\
This paper presents the problem of motion planning for a static uncertain environment by taking into consideration both the geometry and the location probabilities functions without inflating the ambient space dimension while guaranteeing convergence. This is done by extending the well known deterministic \emph{Navigation Function} (NF). \\
NF \cite{koditschek1990robot}  is one of the best known method due to its mathematical elegance and simplicity. A NF is a continuous smooth function with zero value at the target point and a unity value on the boundaries of the environment and the obstacles. In order to ensure a solution, the NF critical points are non-degenerated (i.e. there are no plateau areas in which the gradient of the NF vanishes). Other concepts of NF's have been proposed, for example  Lavalle and Konkimalla \cite{lavalle2001algorithms} numerically solve a discrete differential equation to obtain a NF which simultaneously yields both the geometric path and the control signal, while other methods provide these in a two-step procedure. Another important advantage of the NF algorithm is its asymptotic convergence property.
Some researchers attempted to apply the classical NF method to uncertain environments: Palejiya and Tanner \cite{palejiya2006hybrid} apply the NF when certain switching conditions are met. Loizou et al. \cite{loizou2003closed} use a NF in a portion of a configuration space with the convergence property verified through computer simulations. 
To the best of our knowledge, no attempt has yet been made to modify the NF to fit stochastic scenarios without inflating the  ambient space dimension. In this paper we shall extend the concept of NF to static stochastic environments and analytically prove its convergence.

\subsection{Problem formulation}
\label{problem form}
Let $\C$ be a robot configuration space. Assume that $\C$ is a subset of a smooth manifold which is $\R^n$. We denote the robot location by $\x_r\in\C$ and the $i$-th obstacle fixed location by $\x_{i\in\Ob}\in\C$, where $\Ob$ denotes the set of obstacles. Here, the location refers to the center of the object.
Since deterministic knowledge about the $\x_r,\x_{i\in\Ob}$, is not always available, we use a set of probability density functions (PDF). In this paper we assume the distributions are Gaussian and denote a PDF by $p(\x)$, where $\x\sim \N(\hat{\x},\Sigma)$.
Note that while the locations are random variables, the geometries of the robot and the obstacles are perfectly known. 
We assume that $\x_r,\x_{i\in\Ob}$ are estimated using a nonlinear filter, which is a source for uncertainties. In this paper we assume that all locations are estimated using an external tracker so the uncertainties of the locations of the  robot, and the obstacles are independent. Our main problem is formulated as follows:
\begin{quoting}\textbf{\textit{
Given a static environment with a probabilistic density function of the robot's and the obstacles' locations, that characterizes the uncertainties of their localizations’, determine whether convergence of motion planning to the target's configuration ($q_d $) is guaranteed for a given allowable probability for collision ($\Delta$), and if so, generate a path that reduces the probability for collisions. }}
\end{quoting}

Note that we seek a smooth connected path $\pi:[0 ,1] \rightarrow \C$, (in what follows we shall track the steepest descent curve of a smooth function $\varphi:\C \rightarrow \R$).
Here, $\Delta$ indicates the highest allowable probability for collision (see for example \cite{pahlajani2014error} and \cite{petersen1998optimal}). It is expected that in some scenarios the robot would follow a shorter path at the expense of the collision probability. Thus, $\Delta$ limits the probability for collision to a user-determined value.\\
Next, assume the obstacles do not intersect even when taking a dilated radius $R_\Delta$ around each obstacle, which encloses $\Delta$ probability for collision (the curve $\Psi$ in Eq.\ref{num_conv}). One can apply a deterministic NF (where all obstacles are dilated by a constant radius $R_\Delta$, c.f.  \cite{loizou2003closed}) to the problem. However, the proposed PNF considers the uncertainty of each obstacle as illustrated in Fig.\ref{fig:NF_vs_PNF}. In this figure the uncertainty of the right obstacle's location is larger than the uncertainty of the left obstacle's location. The PNF considers these different uncertainties, and therefore, we anticipate that it will provide a safer path compared with the deterministic NF approach for uncertainties. 

\begin{figure}[h]
\centering
\includegraphics[width=0.85\linewidth]{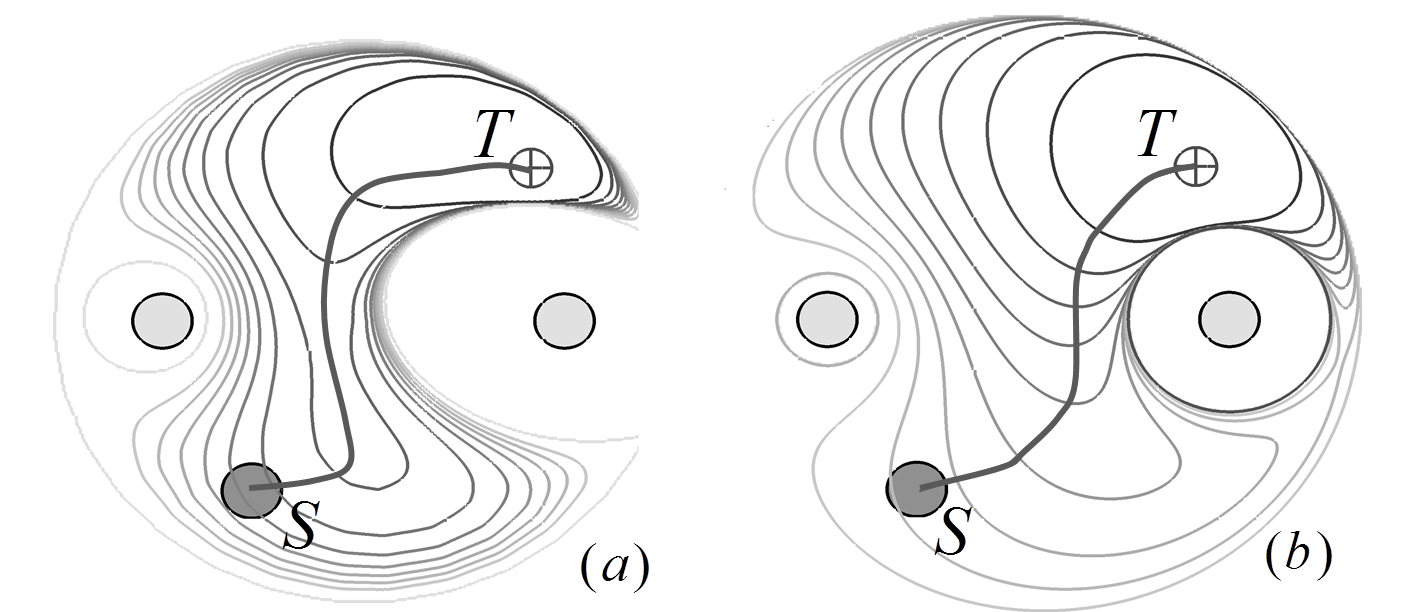}
\caption{A comparison between paths in the PNF (a) and in the NF with a $R_\Delta$ inflated radius (b).$S$ represents the starting point and $T$ the target and  the bold lines indicates the. The light solid discs represent the obstacles and the dark  disc the robot. Here, the STD of the right obstacle is larger than that of the left. While the NF considers the distance from the inflated boundaries, the PNF considers the probability for collision.}
\label{fig:NF_vs_PNF}
\end{figure}

While the theoretical scope of the paper is valid for all dimensions, the mapping of a spherical obstacle from the workspace to the $\C$ may be complex. However, such a description is suitable for a large set of practical scenarios: (1) spatial mobile robots, such as unmanned aerial vehicles  (2) $n$-dimensional serial robots with point obstacles, located "far enough" from the base joint \cite{farber2008invitation}, (3) spider-like planar robots with point obstacles near the end-effector \cite{shvalb2005configuration}.

\section{A Probability density function for collision}
We apply a modified NF in order to incorporate the position uncertainty of the robot and the obstacles, and call this function the \emph{Probability Navigation Function} (PNF) or the \emph{Stochastic Navigation Function} (SNF).
The PNF describes the probability for the robot to collide with an obstacle at a given point, as well as the distance to the target position. In order for the algorithm to be as realistic as possible, the robot and the obstacles possess finite disc shapes (rather than being a point mass).  \vspace{1mm}\\
The shapes of the robot and the obstacles are described by a probability map, and the path is then generated as the PNF gradient.
A common technique used when dealing with motion planning problems (see \cite{ge2006autonomous} \S 10 for extended discussion), is to define the free configuration space $\C_{free}$ (i.e. a subset of $\C\subset \R^n$ where the robot can travel without colliding with obstacles, excluding the boundary). A prevalent method is to define $\C_{free}$ as the complement of $\C_{obs}$, the union of the \emph{Minkowski sums} of the robot with the set of obstacles. Intuitively, the obstacles in the $\C$ space are expended by the robot's volume, while the robot is taken as a point mass. The set of vectors defining the robot's geometry, are measured from its center of mass to any point on the robot body and are denoted by $A$. The set of vectors defining the geometry of \textbf{all obstacles} measured from the origin to their body points are donated by $B$. Thus, we can write:
\begin{equation}
\label{mink}
\C_{obs}=\B \ast  (-\A) = \left\{ {b - a| a \in \A,b \in \B} \right\}
\end{equation}
(we use $\ast $ to denote both Minkowski sum and Convolution operation).  
Note that in order to measure the distance of a point inside the robot from a point inside the obstacle, one should first rotate the robot by $180^o$ (the minus sign in Eq. \ref{mink}). The sets $\A, \B$ are sub spaces of $\C$ making  $\B * (-\A)$  large. One way to overcome this is to confine calculations to an intermediate time step (i.e. $\A,\B \subset \W \subseteq \R^n$). Fig.\ref{fig:Minkowski} demonstrate this process. \\
\begin{figure}[h]
\centering
\includegraphics[width=0.8\linewidth]{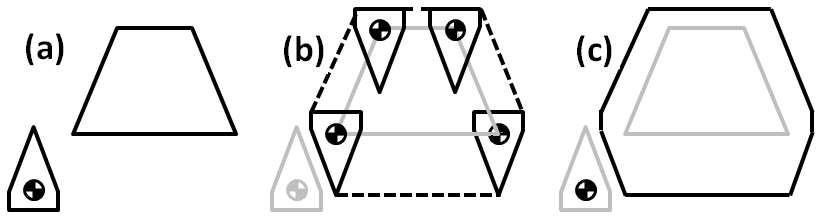}
\caption{The Minkowski sum of (a) a pentagon shaped robot and a trapezoid shaped obstacle. (b) The rotation of the robot by $180^o$. Here, the shortest distance from the low left corner of the obstacle to the robot edge is the same as the shortest distance from the center of the robot to the edge of the obstacle after summation. (c) The resulting  point mass robot and the inflated obstacle.}
\label{fig:Minkowski}
\end{figure}
We shall now incorporate the geometries of the robot and the obstacles together with their location probabilities over three stages (which correspond to \S\ref{geo_conv}, \S\ref{convv} and \S \ref{Convolution of probability}):
\subsection{Convolution of the obstacle's geometry with the robot's \mbox{geometry}.}
\label{geo_conv}
Let us define the disc geometry function as:

$$
 D(\x,\x_c,r) \triangleq \left\{ \begin{array}{l}
 1 ; \ \ \left\| {\x - \x_c} \right\| \le  {r} \\
 0;\ \ otherwise \\
 \end{array} \right. 
$$
Accordingly, define \mbox{$rob\left( x \right)\triangleq D(\x,\x_{r},r_{r})$} as the robot's geometry. Similarly, define the $i$-th obstacle's geometry as \mbox{$ obs_i\left( \x \right)\triangleq D(\x,\hat \x_{o}^i,r_{o})$}.
The Minkowski sum of both functions is denoted by:  
  $$\widetilde{obs_i}\left( \x \right) \triangleq rob\left( \x \right) *  obs\left( \x \right) = D(\x,\hat \x_{o}^i,R) $$ 
where $\x$ is a point in $\C$, and $\hat \x_{o}^i$ the estimated location of the obstacle's center. The estimated location of the robot is  $ \hat \x_{r}$ and $R$ is the radii sum of the robot $r_{r}$ and the obstacle $r_o$. The new robot geometry function is now:
$\widetilde{rob}(\x)=\delta \left( {\x -  {\hat{\x}_{r}}}\right) $, here $\delta$ stands for Dirac's delta function.
\subsection{Convolution of Gaussian functions.}
\label{convv}
To implement the above to a stochastic scenario, let us first consider point-mass obstacles and a point-mass robot with given probability density functions embedded in $\R^n$ for arbitrary configuration space dimension $n$. Eq. \ref{mink} defines a map $f_i: \C \rightarrow \R$.
That is, the Minkowski sum is replaced by a (continuous) convolution of the probability functions. Thus, 
\begin{equation}
\label{gaussian_conv}
{f_i}(\x)  = p\left(  {\x|\hat \x_{o}^i, \Sigma _{o}^i} \right)\ast  p\left( {\x|{\bf 0},{\Sigma _r}} \right)
\end{equation}
following \cite{vinga2004convolution}, Eq. \ref{gaussian_conv} results in the probability function of the $i$-th obstacle location:
\begin{equation*}
{f_i}(\x) = \frac{{{\left( {2\pi } \right)}^{- \frac{n}{2}}}}{{{{\left| {\Sigma _{o}^i + {\Sigma _r}} \right|}^{\frac{1}{2}}}}}{e^{-\frac{1}{2}{{\left( {\hat \x_{o}^i - \x} \right)}^T}{{\left( {\Sigma _{o}^i + {\Sigma _r}} \right)}^{ - 1}}\left( {\hat \x_{o}^i - \x} \right)}}
\end{equation*}
with expectation $\hat \x_{o}^i$ and covariance $\Sigma_i={\Sigma _{o}^i + {\Sigma _r}}$. We denote the distributions for the location of the robot and the locations of the obstacles by:
\begin{equation}
\label{xrobot}
{\x_{r}} \sim \N\left( {{{\hat \x}_{r}},0 } \right)=\delta \left( {\x -  {\hat{\x}_{r}}}\right)
\end{equation}
\begin{equation}
\label{obs_dist}
\x_{o}^i \sim \N\left( {{\hat \x}_{o}^i},\left( {\Sigma _r} + {\Sigma^i_{o}} \right)  \right )= p\left( \x | {{\hat \x}_{o}^i},\Sigma_i \right)
\end{equation}
\subsection{Convolution of probability density function and geometry functions.}
\label{Convolution of probability}
Note that in Eq. \ref{xrobot}, the robot and the obstacles are represented by a point- mass. To extend this we shall now investigate the probability for a collision of a disc shaped obstacle  with a point mass robot (as is often done in motion planning problems).\\
The location of any point $\vv$ of the obstacle relative to a fixed point on the obstacle (e.g. its center of mass) is a  deterministic value. Therefore the location can be defined as a \emph{constant random variable} (see \cite{hodges1970basic}\S 5) by the probability function: ${p_{\vv} }\left( {\x|{\x^i_{o}}} \right) = \delta \left( {\x - \left( {{\x^i_{o}} + \vv } \right)} \right)$.  Note that  the $i$-th obstacle center location $\x^i_{o}$ measured in a global coordinate system, and  $ \vv$ which is measured in a local coordinates system, are independent. The convolution of these functions, which yields the probability distribution function for an infinitely small portion $\vv\in \widetilde{obs_i}(\x)$, is:
\begin{equation*}
 {p_{\vv} }\left( \x \right) = {p_{\vv} }\left( {\x|{\x^i_{o}}} \right) \ast  {f_i}(\x)
\end{equation*} 
Applying \cite{strichartz2003guide} (cf. pg. 53)  yields:
\begin{equation*}
{p_{\vv} }\left( \x \right) = {f_i}\left( {\x - \left( {{{\hat \x}^i}_{o} + {{\vv} }} \right)} \right)
\end{equation*}
which is the PDF for a cell point $\vv$ of the obstacle to be at $\x$. 

Recall that the robot center is at $\x$, we would like to avoid collision  with any of the obstacle's points, so the density function for such a collision is given by the integral:
\begin{equation*}
{p_{tot}}\left(\x,R\right)=\frac{{1}}{A} \int\limits_{{\R^n}} {\widetilde{obs_i}\left( {\vv},R \right){f _i}(\x-\vv)d{{\vv} }}	
\end{equation*}
where $A$ is a normalizing scale factor, selected such that $\int p_{tot}(x)dx=1$ (see \S \ref{Minimal_permited}).\\
Note that $p_{tot}$ concerns only one obstacle. Since multiple obstacles are involved, we shall combine all the corresponding functions $p_{tot}$ when defining the probability NF (see \S \ref{PNF}).\\
For further clarification, note that  the probability for a collision of the robot with an obstacle estimated to be at $\hat{\x}_o$  is  given by the integral:
   \begin{equation*}
   Pr(\|\hat{\x}_r-\hat{\x}_o\|<R)=\int\limits_{\|\hat{\x}_r-\hat{\x}_o\|<R} p_{tot}(\hat{\x}_r,R)
   \end{equation*}
One can think of the convolution operator as locating the $i$-th obstacle at the origin so ${{\hat \x}^i}_{o}=0$, while moving $f_i$ around. This means that $v$ can be considered here to be in either a global or a local coordinates system. Therefore it is easy to see that:
\begin{equation}\label{P_total}
{p_{tot}}\left( \x ,R\right) =\frac{{1}}{A} \widetilde{obs_i}\left( \x,R \right) \ast  {f_i}\left( \x \right)
\end{equation}

\subsection{Convolution of $n$ dimensional disk with Gaussian distribution.} 
\label{N-disc}
We now focus on a specific case applying the analysis introduced in \S \ref{Convolution of probability} above.  Consider Eq. \ref{P_total} \--\ a convolution of a normal distribution $G(\rr)$, where  $\rr\in \R^n$ (with a diagonal covariance of the form: $\Sigma=\sigma I $) and a disc $D(\rr,R)$ :
\begin{equation}
\label{C_r}
C(\rr) = D(\rr,R) * G(\rr)\
\end{equation}
For an arbitrary Gaussian, $\Sigma$ can be taken as a diagonal matrix with all its entries equal to the maximal eigenvalue of the covariance matrix.\\
Assume the disc is centered at the origin and the Gaussian is at $\g\in \R^n$:
\begin{equation*}
    G\left( {\rr-\g} \right)={\left( {\frac{1}
  {{2\pi \sigma}}} \right)^{\frac{n}
  {2}}}{e^{ - \frac{1}{2\sigma} ||\rr - \g|{|^2}}}
\end{equation*}
\cite{plesser2002the} considers the convolution of a disk centered about the origin in $\R^2$ with a Gaussian centered about an arbitrary point. We now generalize Plesser's results, for arbitrary Euclidean ambient space and arbitrary $\sigma$. Eq. \ref{C_r}  may be formulated as:
\begin{equation*}
   C(\g) = \int\limits_{{\mathbb{R}^n}} {G\left( {\rr - \g} \right)D\left( {\rr},R \right)d\rr}
\end{equation*}
The Jacobian for the polar form of the above is \mbox{$J = {r^{n - 1}}\prod\limits_{k = 1}^{n - 2} {{{\sin }^k}\left( {{\phi _{n - 1 - k}}} \right)} $}  (see \cite{NES__blumenson1960derivation}, Pg. 65-66) and thus:
\begin{equation*}
\begin{aligned}
C\left( {\g} \right) =& \int\limits_0^R \frac{2\pi {r^{n - 1}}}{\left( {2\pi \sigma } \right)^{ \frac{n}
{2}}}\left[ {\prod\limits_{k = 1}^{n - 3} {\int\limits_0^\pi  {{{\sin }^k}\left( {{\phi _{n - 1 - k}}} \right)d{\phi _{n - 1 - k}}} } } \right]\cdot\\
&\cdot \int\limits_0^\pi  {{e^{\frac{{ - {r^2} - {g^2} + 2rg\cos \left( {{\phi _1}} \right)}}
{{2\sigma }}}}{{\sin }^{n - 2}}\left( {{\phi _1}} \right)d{\phi _1}} dr 
\end{aligned}
 \end{equation*}
which may be rewritten as:
\begin{equation*}
\begin{aligned}
C & \left( {\g} \right) = \frac{{e^{\frac{{ - {g^2}}}
{{2\sigma }}}}}{{\left( {2\pi \sigma } \right)^{  \frac{n}
{2}}}}\int\limits_0^R 2{\pi ^{n/2}}{r^{n - 1}}{e^{\frac{{ - {r^2}}}
{{2\sigma }}}} \cdot\\
&\cdot \left[ {\frac{1}
{{\sqrt \pi  \Gamma \left( {\left( {n - 1} \right)/2} \right)}}\int\limits_0^\pi  {{e^{\frac{{rg\cos \left( {{\phi _1}} \right)}}
{\sigma }}}{{\sin }^{n - 2}}\left( {{\phi _1}} \right)d{\phi _1}} } \right]dr
\end{aligned}
\end{equation*}
where  $r=||\rr||$  and $g=||\g||$.\\
Following Abramowitz (\cite{abramowitz1972handbook} Eqs. 9.6.10 , 9.6.18):  
\begin{equation*}
\begin{aligned}
C\left( {\g} \right)& = {\left( {2\sigma } \right)^{ - \frac{n}
{2}}}{e^{\frac{{ - {g^2}}}
{{2\sigma }}}}\sum\limits_{k = 0}^\infty  {{\left( {\frac{g}
{{2\sigma }}} \right)}^{2k}}\frac{1}
{{k!}}\cdot \\
& \cdot\left[ {\frac{1}
{{\Gamma \left( {k + n/2} \right)}}\int\limits_0^{{R^2}} {{{\left( {{r^2}} \right)}^{k + \frac{n}
{2} - 1}}{e^{\frac{{ - {r^2}}}
{{2\sigma }}}}d{r^2}} } \right] 
\end{aligned}
\end{equation*}
Recall that $P\left( {a,b} \right) = \frac{1}
   {{\Gamma \left( a \right)}}\int\limits_0^b {{e^{ - x}}{x^{a - 1}}dx} $ is the \emph{Normalized Incomplete Lower Gamma Function}. Rearranging terms results with the equality: 
\begin{equation*}
\frac{1}{{\Gamma (s)}}\int\limits_0^x {{e^{ - t/a}}{t^{s - 1}}dt}  = 
\frac{{{a^{s - 1}}}}
{{\Gamma (s)}}\int\limits_0^{x/a} {{e^{ - z}}{z^{s - 1}}adz}  = P\left( {s,\frac{x}
{a}} \right){a^s}
\end{equation*}
Finally, for an $n$-dimensional disk-shaped obstacles distributed normally,  Eq. \ref{P_total} becomes:
\begin{equation}
\label{P_total2}
\begin{aligned}
{p_{tot}}\left( {\x,R,\sigma } \right)& =  \\
{e^{\frac{{ - {{\left\| \x \right\|}^2}}}
{{2\sigma }}}}&\sum\limits_{m = 0}^\infty  {{{\left( {\frac{{{{\left\| \x \right\|}^2}}}
{{2\sigma }}} \right)}^m}\frac{1}
{{m!}}P\left( {m + \frac{n}
{2},\frac{{{R^2}}}
{{2\sigma }}} \right)}
\end{aligned} 
\end{equation}
where $\x$ is the location vector.  
\subsection{The Gradient and the Hessian of $p_{tot}$.} 
We now calculate the Gradient and Hessian of \ref{P_total2}:
  \begin{equation*} \begin{aligned}
  \nabla {p_{tot}}(\q,R,\sigma ) &= \frac{{\partial {p_{tot}}(\q,R,\sigma )}}
 {{\partial ||\q||_2^2}}\frac{{\partial ||\q||_2^2}}
 {{\partial \q}} =\\
  = 2\q{e^\frac{ - {{\left\| \q \right\|}^2} }{2\sigma}}&\left( \sum\limits_{m = 1}^\infty  \frac{{{\left\| \q \right\|}^{2\left( {m - 1} \right)}}}{{{\left( {2\sigma } \right)}^{m} \left( {m - 1} \right)! }}
 P\left( m + \frac{n}{2},\frac{R^2}{2\sigma } \right) \right.- \\
 & \left.- \sum\limits_{m = 0}^\infty  {{{\left( {2\sigma } \right)}^{ - m - 1}}\frac{{{\left\| \q \right\|}^{2m}}}
 {{m!}}P\left( {m + \frac{n}
 {2},\frac{{{R^2}}}
 {{2\sigma }}} \right)}  \right) 
  \end{aligned}\end{equation*}
 Following Gautschi \cite{gautschi1999note} we know that: $P\left( {a + 1,x} \right) - P\left( {a,x} \right) =  - \frac{{{x^a}{e^{ - x}}}}{{\Gamma \left( {a + 1} \right)}}$, and since the modified Bessel function can be written as: $ {I_a}\left( x \right) = \sum\limits_{m = 0}^\infty  {\frac{1}{{m!\Gamma \left( {m + a + 1} \right)}}{{\left( {\frac{x}{2}} \right)}^{2m + a}}}  $ 
 \begin{equation}
 \label{grad_p_tot}
 \nabla {p_{tot}}(\q,R,\sigma ) = \frac{{ - 2\q{R^n}{e^{ - \left( {\frac{{{R^2} + {{\left\| \q \right\|}^2}}}
 {{2\sigma }}} \right)}}}}
 {{{{\left( {2\sigma } \right)}^{\frac{n}
 {2} + 1}}}}{I_0}\left( {\frac{{\left\| \q \right\|R}}
 {\sigma }} \right)
 \end{equation}
 The Hessian is:
 \begin{equation}
 \label{hessian_p_tot}
 \begin{aligned}
 {\nabla ^2}{p_{tot}}&(\q,R,\sigma )= {e^{ - \frac{{{{\left\| \q \right\|}^2} + {R^2}}}{{2\sigma }}}}\frac{{\q{{\q}^T}{R^n}}}{{{2^{\frac{n}
{2}}}{\sigma ^{\frac{n}{2} + 2}}}}\cdot \\
& \cdot \left( {{I_0}\left( {\frac{{\left\| \q \right\|R}}{\sigma }} \right) - \frac{R}{{\left\| \q \right\|}}{I_1}\left( {\frac{{\left\|\q \right\|R}}{\sigma }} \right)} \right)
 \end{aligned}
 \end{equation}
\subsection{Minimal permitted collision probability}
\label{Minimal_permited}
In order to ensure a reasonably safe movement, we limit the maximal collision probability to a predefined value $\Delta$. In other words, we are interested in a closed curve $\Psi$ in $\C\subseteq \R^n$  such that:
\begin{equation}
\label{num_conv}
\frac{1}{A}\int_{\Psi} {\widetilde{obs_i}\left( \x ,R \right) \ast  {f_i}\left( \x \right)} d\x = \Delta
\end{equation}
We pursue a \emph{safety} distance $R_{\Delta}$ from $\widetilde{obs_i}$ that will ensure  probability for collision of at most $1-\Delta$:
\begin{equation}
\label{delta_1}
\Delta  = \frac{1}
{A}\int\limits_{{\mathbb{R}^n}} {C\left( \xi  \right)d^n {\xi}}
\end{equation}
where $ \xi = |\rr|$ is the distance from the origin and $A$ is the normalization factor, which by expansion yields:
\begin{equation*}
\Delta  = \frac{1}
{A}\int\limits_{{\mathbb{R}^n}} {{e^{\frac{{ - {g^2}}}
{{2\sigma }}}}\sum\limits_{m = 0}^\infty  {\frac{{{g^{2m}}}}
{{{{\left( {2\sigma } \right)}^m}m!}}P\left( {m + \frac{n}
{2},\frac{{{R^2}}}
{{2\sigma }}} \right)d\vec g} } 
\end{equation*}
Again, using \cite{NES__blumenson1960derivation} results in:
\begin{equation*}
\begin{aligned}
\Delta  = \frac{1}
{A}&\frac{{{\pi ^{\left( {n - 1} \right)/2}}}}
{{\Gamma \left( {\left( {n - 1} \right)/2} \right)}}\sum\limits_{m = 0}^\infty  {P\left( {m + \frac{n}
{2},\frac{{{R^2}}}
{{2\sigma }}} \right)} \frac{1}
{{{{\left( {2\sigma } \right)}^m}m!}} \cdot \\
& \cdot \int\limits_0^{R_\Delta ^2} {{{\left( {{g^2}} \right)}^{m + \frac{n}
{2} - 1}}{e^{\frac{{ - {g^2}}}
{{2\sigma }}}}}  \int\limits_0^\pi  {{{\sin }^{n - 2}}\phi d\phi } d{g^2}
\end{aligned}
\end{equation*}
Denoting  the \emph{Double Factorial} by $X!!$ results in the following equality:
\begin{equation*}l(n) = \int\limits_0^\pi  {{{\sin }^n}\phi d\phi }  = \frac{{\left( {n - 1} \right)!!}}
{{n!!}} \cdot \left\{ \begin{gathered}
  \pi  , \ \ n \ mod \ 2 =0 \hfill \\
  2 , \ \ n \ mod \ 2= 1 \hfill \\
\end{gathered}  \right.\end{equation*}
and  the value of $A$ can be simplified due to Fubini's theorem:
\begin{equation*}
\begin{aligned}
A  = \int\limits_{{\mathbb{\R}^n}} {D(\rr)\ast  G(\rr)} d \rr= \int\limits_{{\R^n}} {D(\rr)} d \rr\int\limits_{{\mathbb{R}^n}} {G(\rr)} d\rr
= \frac{R^n \pi ^{n/2}}{\Gamma(\frac{n}{2}+1)}
\end{aligned}
\end{equation*}
Finally Eq. \ref{delta_1} may be written as:
\begin{equation}
\label{find R_delta}
\begin{aligned}
\Delta  = & \frac{{\Gamma \left( {n/2 + 1} \right)}}
{{\sqrt \pi  {R^n}\Gamma \left( {\left( {n - 1} \right)/2} \right)}} \cdot\\
\cdot &\sum\limits_{m = 0}^\infty  {\frac{1}
{{m!}}P\left( {m + \frac{n}
{2},\frac{{{R^2}}}
{{2\sigma }}} \right)\gamma \left( {m + \frac{n}
{2},\frac{{R_\Delta ^2}}
{{2\sigma }}} \right)} 
\end{aligned}
\end{equation}
where  $ \gamma (a,b) = \int\limits_0^b {{e^{ - t}}{t^{a - 1}}dt}$ , is the \emph{Lower  Incomplete Gamma Function}.
Eq. \ref{find R_delta} can be approximately solved for $R_\Delta$. $p_\Delta $ can be calculated as $p_{tot}(x)$ for (all) $x\in\Psi$ which is a circle of radius $R_{\Delta}$ through the obstacle location (i.e. ${p_\Delta } = p_{tot} \left( R_\Delta,R,\sigma \right)$ see also Eq. \ref{P_total2}). 

\section{Probability Navigation Function}
\label{PNF}
This section presents the approach for generating motion planning in uncertain environments. The following discussion is an extension of the deterministic NF suggested by Rimon and Koditschek \cite{koditschek1990robot}.
Denoting the target position by $q_d$, the NF is defined at a point $q\in \C$ as:
\begin{equation}
\label{NF}
{\varphi _k}(\q) = \frac{\gamma_d^2}{{{{\left[ {\gamma_d^{2k} + \beta \left( \q \right)} \right]}^{\frac{1}{k}}}}}
\end{equation}
where $k$ is a predefined constant which ensures the Morse nature of the function (for both the deterministic and probabilistic NFs). Recall that a real-valued smooth function on a differentiable manifold is called \emph{Morse} if all its critical points are non-degenerated; $k$ will be discussed in Section \ref{sec is NF}.
$\gamma_d(\q)={{||\q - {\q_d}|{|^2}}} $ and $\beta \left( \q \right)$ is:
 \begin{equation}
 \beta \left( \q \right) = \prod\limits_{i = 0}^{{N_{o}}} {{\beta _i}\left( \q \right)}
 \end{equation}
\mbox{where: \small $ {\beta _i}\left( \q \right) = \left\{ \begin{array}{l}
   - ||\q - {\q_0}|{|^2} + \rho _0^2 \ \  ; i = 0 \\
  \ \  ||\q - {\q_i}|{|^2} - \rho _i^2    \ \   ; i > 0 \\
  \end{array} \right.$ } \\ 
  \normalsize
Here $\q_0$ defines the center of the permissible area, considered as the coordinates' origin, while $\q_i$ for all $i \in \Ob$, is the center of the $i$-th obstacle.\\
The numerator of Eq. \ref{NF} is defined in such a way that the robot is attracted to the target position, while the denominator ensures obstacle avoidance.  

Considering a stochastic scenario, we would like to minimize the probability for a collision while maintaining the shortest path to the target.  In the deterministic scenario, $\beta _i\left( \q \right)$   is a function of the distance between $\q$ and the obstacle's boundary. Our goal is to replace $\beta _i$ by a function that is based on the probability for collision at location $\q$. We set a threshold value for collision probability by replacing the obstacles' geometric edge by the edges $\Psi$ \--\ discussed above (see Eq. \ref{num_conv}).\\
\label{Modification}
We then modify $\beta$ to fit an uncertain environment. In a deterministic scenario, $\varphi_K$ decreases the distance to the target position while avoiding the obstacles. In a probabilistic scenario the probability for collision would be limited by a predetermined value- $\Delta$. \\
In order to do so, we replace the original $\beta$ with the probabilistic value $p_i(\q)$ \--\ the probability density function at $q$ (discussed in Section \ref{Convolution of probability}). This equals $p_{tot}(\q-\q_i)$ (see Eq. \ref{P_total2}) computed for the $i$-th obstacle ($i\in \Ob $) and for the workspace boundary:
\begin{equation}
\label{beta 1}
\begin{gathered}
\beta _i (\q) =  p_{\Delta_i} - {p_i}\left( \q \right)\\
 \beta _0 (\q)= - p_{\Delta_0} + {p_0}\left( \q \right) 
 \end{gathered}
\end{equation}
where, $ {p_\Delta}_i=p_{tot}\left ( R_{\Delta_i} ,R_i,\sigma_i\right)$ 
Thus, $\beta_i$ and $\beta$ vanish on the extended boundaries of each obstacle defined by $R_{\Delta_i}$,  i.e. where the probability for collision is $\Delta$ (see Fig.\ref{fig:con1}).  \\
Note that $ p_0 $ and $p_{\Delta_0} $ refer to the external boundary, computed based on the probability density function of the robot, and also that $p_{\Delta_0} $ is computed as in Eq. \ref{find R_delta} replacing $\Delta$ with $1-\Delta$. 
\section{Is $\varphi$ a \textit{Probabilistic Navigation Function}? }
\label{sec is NF}
We start by defining a NF in the context of \cite{koditschek1990robot}:\\
{\bf Definition:}
A map is said to be a \textit{navigation function} if it satisfies the following conditions:

\begin{enumerate}
\item It is \emph{analytic} in all $\q\in\C_{free}$;
 \item  It is \emph{polar} throughout $\C$, with single minimum at $\q_d \in \C_{free}$;
  \item It is \emph{morse} on $\C_{free}$; 
 \item It is \emph{admissible} on $\C_{free}$.
   \end{enumerate}

We now extend the above definition to a stochastic scenario and prove that such a function is indeed a probabilistic NF:\\
{\bf Definition:}
A map $\varphi$ is said to be a \textit{probabilistic navigation function} (PNF) if it satisfies the following conditions:

\begin{enumerate}
\item
It is a NF.
\item 
The probability for collision is bounded by a predefined probability $\Delta$.
\end{enumerate}

Note that as a consequence of the above, following {${\boldmath \nabla \varphi}$ minimizes the probability for collision} (subject to decreasing the distance to the target). 
The NF is the composition:
\begin{equation*}
\varphi=\sigma_d \circ \sigma \circ \hat{\varphi}
\end{equation*}
where: $ \sigma_d(x)=(x)^{1/k} $ ; $ \sigma(x)=\frac{x}{1+x} $ and  $ \hat{\varphi}=\frac{\gamma_d^k}{\beta} $ and $\circ$ is the composition operator. 
In this paper we change only $\hat{\varphi}$. According to proposition 2.7 in \cite{koditschek1990robot} it suffices to verify the first condition- (1) only for  $\hat{\varphi}$ (note that the forth requirement directly follows from the definitions).\\
\begin{figure}[h]
\centering
\includegraphics[width=0.55\linewidth]{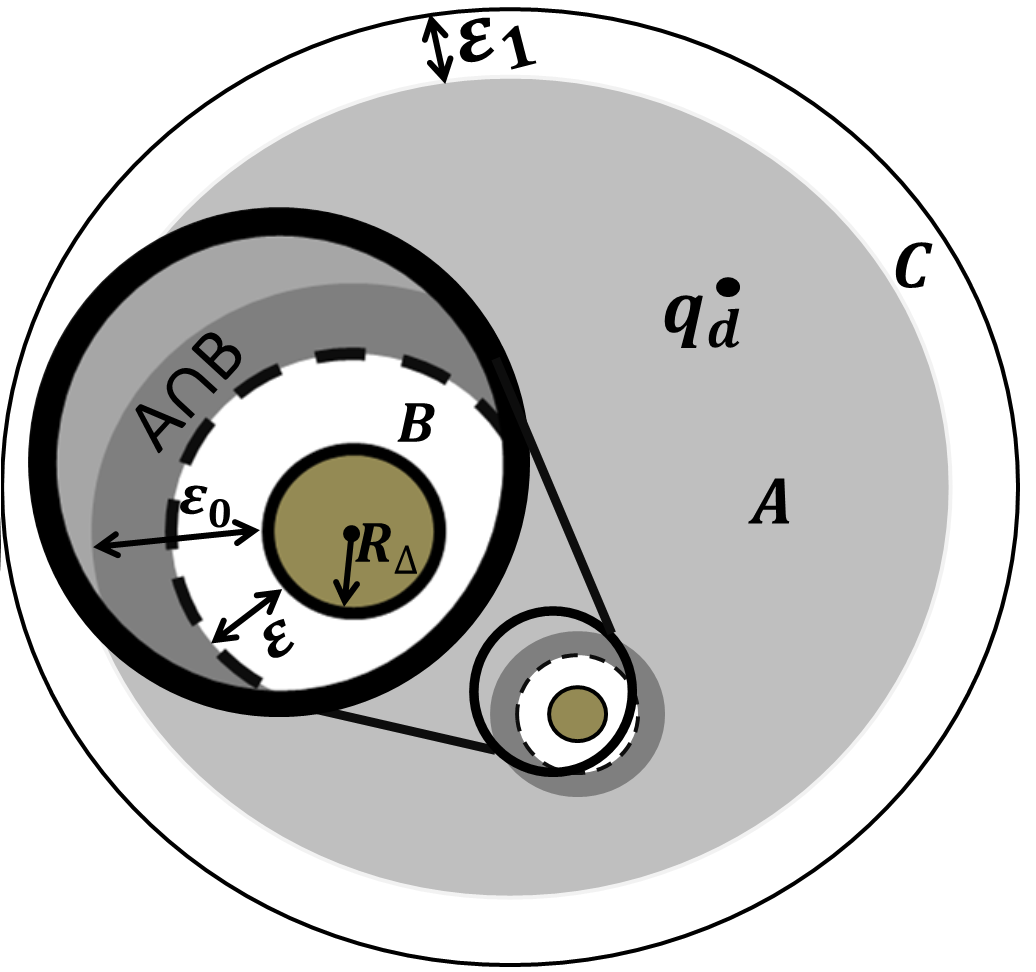}
\caption{The partition of the configuration space: (A) is the region which extends away from the obstacles such that $\beta_i\ge \varepsilon$ and stretches to the configuration space boundary (B) is the region which extends from the obstacles' boundaries. and away from it up to  $\beta_i\le\varepsilon_0$, (C) is the region which extends up to $\beta_0\le\varepsilon_1$ away from  the configuration space boundary. $R_\Delta$ indicates the safety obstacle radius, see Subsection \ref{Minimal_permited}.}
\label{fig:con1}
\end{figure}
We shall now prove that $\hat{\varphi}$ constitutes a NF. In Proposition \ref{prop a} we will prove that $\varphi$ attains a minimum value at the destination $q_d$. In order for our motion planning scheme to converge, $\varphi$ must not have critical points on $\partial \C_{free}$ (i.e. points  where the gradient vanishes), which we shall prove in Proposition \ref{prop b} that results in the interior of $\C_{free}$.\\
For convergence we require that all critical points in $\C_{free}$ are non-degenerated. We refer to this region as "near the \mbox{$i$-th} obstacle" and denote it by $\B_i(\varepsilon)=\{ \q|\ 0<\beta_i(\q)<\varepsilon\}$. Since the obstacles do not intersect, there exist $\varepsilon>0$ such that $\beta_i(\q)\cap\beta_i(\q)=\emptyset$ for all $i\ne j \in\Ob$ and $\beta_i(\q)\cap \q_d=\emptyset$ for all $i\in \Ob$. In other words, we need to prove that:
\begin{equation*}
\C_{free}=\{\q | \beta_i(\q)\geq \varepsilon , \forall i\in\Ob\} \cup  {\B_0}\left( \varepsilon  \right)\cup \bigcup\limits_{i = 1}^{{N_{o}}} {{\B_i}\left( \varepsilon  \right)} 
 \end{equation*} 
has no non-degenerate critical points in either regions (indicated by the three components). Propositions \ref{prop c} and \ref{prop_e} respectively prove that the first and second regions have no critical points, while Proposition \ref{prop_d} proves that all critical points near the obstacles are not local minimum points. In Proposition \ref{prop_f} we conclude that $\varphi$ is a Morse function by showing that the function is non-degenerate near the obstacles.
\begin{pro}
\label{prop a}
The destination region located at $ \q_d $ is a local minimum of $ {\varphi} $. 
\end{pro}
\begin{proof}
This is identical to the proof of Proposition 3.2 in \cite{koditschek1990robot}.
  \begin{flushright}   \small$\blacksquare$ \normalfont\end{flushright}  
\end{proof}
For the following discussion, we denote:
${\bar \beta }_i=\prod\limits_{j=0, j\ne i}^{N_{o}}\beta_j$
\begin{pro}
\label{prop b}
All the critical points of $ \varphi $ are in the interior of  $ \C_{free} $. 
\end{pro}
\begin{proof}
We focus our attention on some point $\q'\in\partial \C_{free}$. Obviously $\beta_i=0$ for a certain $i \in \Ob$,  and $\beta_j>0$ for the rest $j\ne i$.
Differentiating yields:
\begin{equation*}
\nabla \varphi \left( {\q'} \right) = {\left. {\frac{1}
{{{\gamma _d}}}\left( {\nabla {\gamma _d} - \frac{1}
{k}\gamma _d^{1 - k}\left( {k\gamma _d^{k - 1}\nabla {\gamma _d} + \nabla \beta } \right)} \right)} \right|_{\q'}}=
\end{equation*}
\begin{equation*}
=  - \frac{1}
{{k\gamma _d^k}}\nabla {\beta _i}{{\bar \beta }_i} \ne 0
\end{equation*}
which proves the proposition since:
 \begin{equation*}\nabla {\beta _i}\left( {\q'} \right) =  - \nabla {p_{tot}}\left( {\q' - {\q_i},{R_i},{\sigma _i}} \right) \ne 0\end{equation*}
 \begin{flushright} \small$\blacksquare$ \normalfont\end{flushright}  
\end{proof} 
As $k$ increases, the  critical points of $\hat{\varphi}$ approach those of $\gamma_d$. We show this by proving that there are no critical points far away from the obstacles:
\begin{pro}
\label{prop c}
For every $\varepsilon>0$ there exist $N(\varepsilon)\in \R$ such that for all $k \geq N(\varepsilon)$, $\varphi$ has no critical points in $\{\q | \beta_i(\q)\geq \varepsilon , \forall i\in\Ob\}$. 
\end{pro}
\begin{proof}
Note that if $\hat{\varphi}$ has no critical points at a given region, neither will $\varphi$. Thus we prove the proposition for $\hat{\varphi}$.  \\
A critical point satisfies:
\begin{equation*} \nabla \left( \hat \varphi  \right) = \frac{{\gamma _d^{k - 1}\left( {k\beta \nabla {\gamma _d} - {\gamma _d}\nabla \beta } \right)}}{{{\beta ^2}}} = 0 \end{equation*} 
so:
 \begin{equation}
  \label{k beta}
  k\beta \nabla {\gamma _d} = {\gamma _d}\nabla \beta
 \end{equation}
Taking the magnitude of Eq. \ref{k beta} yields:
\mbox{$ k\beta \left\| {\nabla {\gamma _d}} \right\| = {\gamma _d}\left\| {\nabla \beta } \right\|$} 
To avoid a critical point we require:
 $k > \frac{\gamma _d{\left\| {\nabla \beta } \right\|}}
{\left\| {\nabla {\gamma _d}} \right\| \beta } $ \\
Since, $\nabla\beta=\sum\limits_{i=0}^{N_{o}}\nabla\beta_i \bar{\beta}_i$,  $\left\|\nabla \gamma_d \right\|=2 \sqrt{\gamma_d}$ and \mbox{$\frac{\beta}{\bar{\beta}_i} = \beta_i \geq \varepsilon$}, the parameter $k$ must comply with the following constraint:
\begin{equation}
\label{N epsilon}
k \geqslant \frac{1}
{{2\varepsilon }}\max\limits_{\C}  \{{\sqrt {{\gamma _d}} }\} \sum\limits_{}^{} {\max\limits_{\C}  \{{\| {\nabla {\beta _i}} \|}\} }  \triangleq N\left( \varepsilon  \right)
\end{equation}
with $\max \limits_{\q} \{\gamma_d(\q)\}= R_0+\|\q_d\|$.
  \begin{flushright}   \small$\blacksquare$ \normalfont\end{flushright}  
\end{proof}
\begin{pro}
\label{prop_d}
There exists an $\varepsilon_0 $ such that $\hat \varphi$ has no local minimum in the set $\B_i(\varepsilon), i \in \Ob $ (near the obstacles) for $\varepsilon\leq \varepsilon_0$:  
\end{pro}
\begin{proof}
The NF must "flow" around the obstacles. We therefore, show that at least one eigenvalue of $\nabla^2 \hat{\varphi}$ is negative by calculating the projection onto the direction perpendicular to the gradient of $\beta_i$ at $\q$. \\
Consider a critical point $\q_c \in \B_i (\varepsilon)$. The Hessian of $\hat \varphi$ is:
\begin{equation*}
\begin{aligned}
&{\nabla ^2}\hat \varphi \left( \q \right) = \frac{1}
{{{\beta ^2}}}\left( {\beta {\nabla ^2}\gamma _d^k - \gamma _d^k{\nabla ^2}\beta } \right) =\\ 
=&\frac{{\gamma _d^{k - 2}}}
{{{\beta ^2}}}\left( {k\beta \left( {{\gamma _d}{\nabla ^2}{\gamma _d} + \left( {k - 1} \right)\nabla {\gamma _d}\nabla \gamma _d^T} \right) - \gamma _d^2{\nabla ^2}\beta } \right)
\end{aligned}
\end{equation*}
Taking the tensor product of both sides of Eq. \ref{k beta} yields:
\begin{equation*}
{\left( {k\beta } \right)^2}\nabla {\gamma _d}\nabla \gamma _d^T = \gamma _d^2\nabla \beta \nabla {\beta ^T}
\end{equation*}
So, the Hessian of $\hat \varphi$ becomes:
\begin{equation}
\label{Hessian_varphi}
{\nabla ^2}\hat \varphi \left( \q \right) =
\frac{{\gamma _d^{k - 1}}}
{{{\beta ^2}}}\left( {k\beta {\nabla ^2}{\gamma _d} +\frac{k-1}{k} \frac{{{\gamma _d}}}
{\beta }\nabla \beta \nabla {\beta ^T} - \gamma _d^2{\nabla ^2}\beta } \right)
\end{equation}
Let us denote	 $A_s\triangleq \frac{1}{2}(A+A^T)$ \--\ the symmetric part of the matrix $A$, so we can write:
\begin{equation*}
\begin{aligned}
{\nabla ^2}&\hat \varphi ( \q ) = \frac{{\gamma _d^{k - 1}}}
{{{\beta ^2}}}\left( k\beta {\nabla ^2}{\gamma _d} + \left( 1 - \frac{1}{k} \right) \cdot \right. \\
  \cdot \frac{{{\gamma _d}}}{\beta }&\left( {\beta _i^2  \nabla {{\bar \beta }_i}\nabla {{\bar \beta }_i}^T + 2{\beta _i}{{\bar \beta }_i}{{\left( {\nabla {{\bar \beta }_i}\nabla {\beta _i}^T} \right)}_s} + {{\bar \beta }_i}^2\nabla {\beta _i}\nabla {\beta _i}^T} \right) - \\ 
& \left. - {\gamma _d}\left( {{\beta _i}{\nabla ^2}{{\bar \beta }_i} + 2{{\left( {\nabla {{\bar \beta }_i}\nabla {\beta _i}^T} \right)}_s} + {{\bar \beta }_i}{\nabla ^2}{\beta _i}} \right) \right)
\end{aligned}
\end{equation*}
Note that $\nabla {\beta ^T}\hat v = {{\hat v}^T}\nabla \beta  = 0 $,  and $\nabla^2\gamma_d=2I $. Taking the quadratic form of $\hat{\varphi}$ by an arbitrary orthogonal vector to  $\nabla \beta_i$:
 $\hat v \triangleq \frac{{\nabla {\beta _i}\left( {{\q_c}} \right)}}
{{\left\| {\nabla {\beta _i}\left( {{\q_c}} \right)} \right\|}} \bot$  
we can write:
\begin{equation}
\label{v nabla phi v}
\begin{aligned}
\hat v^T{\nabla ^2}&\hat \varphi \left( \q \right)\hat v =  \gamma _d \bar{\beta}_i v^T \nabla^2\beta_i v + \\
+ {\beta _i} & \hat v^T  \left(2 k \beta I +  \left(1-\frac{1}{k}\right)\beta_i \nabla \bar \beta_i\nabla \bar \beta_i^T +\gamma_d \nabla^2 \bar{\beta}_i   \right)\hat v  
\end{aligned}
\end{equation} 
It is hard to conclude whether the second component is positive or not. 
But note that the Hessian of $\beta_i$ (see Eq. \ref{hessian_p_tot})
\begin{equation*}
{\nabla ^2}{\beta _i} =  - {\nabla ^2}{p_{tot}}(\q - {\q_i},{R_i},{\sigma _i})
\end{equation*}
is negative definite since, $I_0(x)>I_1(x) \ \forall x $ and $\|\q-\q_i\|>R_i \ \forall q \in \B_i(\varepsilon)$. Additionally, both $\gamma _d$  and  ${\bar \beta }_i$ are positive, therefore the second term is negative. 
\\
To ensure that Eq. \ref{v nabla phi v} is negative we can bound $\beta_i$ with $\varepsilon$ by: \small

 $$
\varepsilon_<  \varepsilon_0 \triangleq  
   \frac{\min\limits_{\q \in \B_i(\varepsilon)} \{ \gamma _d \bar{\beta}_i v^T \nabla^2\beta_i v \}}  {\max\limits_{\q \in \B_i(\varepsilon)} \{\hat v^T  \left(2 k \beta I +  \left(1-\frac{1}{k}\right)\beta_i \nabla \bar \beta_i\nabla \bar \beta_i^T +\gamma_d \nabla^2 \bar{\beta}_i   \right)\hat v   \} } 
$$  \normalfont
See Lemmas \ref{lemma_max_nab_beta_i}, \ref{max_bar_beta} and \ref{max_nabla_bar} in the Appendix for explicit expressions of the extermal terms.
 \begin{flushright}   \small$\blacksquare$ \normalfont\end{flushright}  
\end{proof}
\begin{pro}
\label{prop_e}
If $ k \geqslant N(\varepsilon)$ then there exists an $\varepsilon_1>0 $ such that $\hat{\varphi}$ has no critical points near the workspace boundary, as long as $\varepsilon \leq \varepsilon_1$.
\end{pro}

\begin{proof}
The inner product:
\begin{equation*}
\nabla\hat{\varphi}\nabla\gamma_d =\frac{\gamma_d^k}{\beta^2}\left( 4k\beta- \nabla\beta \nabla\gamma_d\right) 
 > \beta_0 \frac{\gamma_d^k}{\beta^2}\left(4k\bar{\beta}_0 - \nabla\bar{\beta}_0\nabla\gamma_d \right) 
\end{equation*}
according to Eq. \ref{N epsilon}: $
\beta_0 \frac{\gamma_d^k}{\beta^2}\left(4k\bar{\beta}_0 - \nabla\bar{\beta}_0\nabla\gamma_d \right) >0. $ \\
To estimate the second term, define $\varepsilon_1$ as the probability for a robot located at $\q_d$ to collide with the workspace boundary.
 \begin{equation*}
\varepsilon_1 \triangleq p_{\Delta_0}-p_{tot}(\q_d,R_0,\sigma_{r})
\end{equation*}
$\beta_0$ is restricted by:  $ 
\beta_0=p_{\Delta_0}-p_{tot}(\q,R_0,\sigma_{r})<\varepsilon_1$

This is valid since all points in $\B(\varepsilon_0)$ are closer to the boundary than $q_d$.  $\nabla \beta_0$ points away from the destination $q_d$ at any point $q$ in $\B(\varepsilon_0) $ since \mbox{$\nabla\beta_0= -\nabla p_{tot}\left(\q,R_0,\sigma_{r}\right)$} (see Eq. \ref{grad_p_tot}), and \mbox{$\nabla \gamma_d = 2\|\q-\q_d\|>0$}, so  \mbox{$\nabla \gamma_d \nabla\beta_0<0 $}.
\indent This completes the proof. 
 \begin{flushright}   \small$\blacksquare$ \normalfont\end{flushright}  
\end{proof}
We showed that near the obstacles there may be critical points of $\hat\varphi$. We also proved that such points will have a negative gradient component directed tangentially to the obstacles.
Yet, in order for  $\hat\varphi$ to be a NF we need to show that it is a Morse function.\\
\begin{pro}
\label{prop_f}
$\varphi$ is a Morse function.
\end{pro}

\begin{proof}
We would like to prove that the component  of the gradient of $\hat\varphi$ in the radial direction to the obstacle is positive. This way $\nabla \hat\varphi$ will not have any degenerate direction as required.\\
Substituting \ref{k beta} into Eq. \ref{Hessian_varphi} and multiplying both sides of the equation by: $ \tilde{v} \triangleq \frac{\nabla\beta_i}{\|\nabla\beta_i\|}$ it becomes:

\begin{equation*}
\begin{aligned}
&\frac{\beta^2}{\gamma_d^{k-1}} \tilde{v}^T \nabla^2 \hat{\varphi}\tilde{v} = \\
= &\frac{\gamma_d}{2k\beta}\|\nabla\beta\|^2 + \left( 1-\frac{1}{k} \right) \frac{\gamma_d}{\beta} \left(\nabla\beta \cdot \tilde{v}\right)^2 - \gamma_d \tilde{v}^T \nabla^2\beta \tilde{v} 
 \end{aligned}
\end{equation*}
Algebraic manipulations lead to (compare with [Prop. 3.9,\cite{koditschek1990robot}]):

\begin{equation*}
\begin{aligned}
\frac{\beta^2}{\gamma_d^{k-1}} \tilde{v}^T \nabla^2 \hat{\varphi}\tilde{v}  
&\geq \frac{\gamma_d}{\beta_i} \left( \left( 1-\frac{1}{k} \right)\bar{\beta}_i \|\nabla\beta_i\|^2 -\right.\\
&\left. - \tilde{v}^T \left(\beta_i^2 \nabla^2\bar{\beta}_i + \beta_i \bar{\beta}_i \nabla^2 \beta_i \right) \tilde{v}  \right) 
 \end{aligned}
\end{equation*}
Since $q \in \B_i(\varepsilon)$, and assuming that $k \geq 2$ it can be rearranged as:
\small
$$
 \begin{aligned}
 \frac{\gamma_d}{\beta_i} \left( \left( \frac{\bar{\beta}_i}{4} \|\nabla\beta_i\|^2   - \varepsilon \bar{\beta}_i \tilde{v}^T \nabla^2 \beta_i  \tilde{v} \right.  \right)  
   \left.  +\left(\frac{\bar{\beta}_i}{4} \|\nabla\beta_i\|^2  -\varepsilon^2 \tilde{v}^T \nabla^2\bar{\beta}_i \tilde{v} \right) \right) 
 \end{aligned}
 $$
\normalsize
For the first term to be positive we require:
 \begin{equation*}
 \varepsilon \leq  \frac{\min\limits_{\q \in \B_i(\varepsilon)} \{ \|\nabla\beta_i \|^2 \}}{4  \max\limits_{\q \in \B_i(\varepsilon)} \{|\nabla^2 \beta_i|\}} \triangleq \varepsilon'
\end{equation*} 
 and a sufficient condition for the second term to be positive  we require:
 \begin{equation*}
 \varepsilon^2 \leq \frac{\bar{\beta}_i\|\nabla\beta_i\|^2}{4 |\tilde{v}^T \nabla^2\beta_i\tilde{v}|} \leq  \frac{\min \{\sqrt{\bar{\beta}_i}\|\nabla\beta_i\|\}}{2 \sqrt{\max \{|\nabla^2\beta_i|\}}} \triangleq \varepsilon''  
\end{equation*}  
See Lemmas \ref{lemma_max_nab_beta_i} and \ref{max_bar_beta} in the Appendix for explicit expressions for the extremal terms. By restricting the distance of $\q$ to the obstacles such that $\beta(\q)<\min \{\varepsilon',\varepsilon''\}$, we guarantee that $\varphi$ is a Morse function.
\vspace{-0.5cm}
\begin{flushright}   \small$\blacksquare$ \normalfont\end{flushright}  
\end{proof}
Finally, in order for $\varphi$ to be a NF in all $\C$, we require that $\varepsilon=min\{ \varepsilon_0,\varepsilon_1,\varepsilon',\varepsilon''\}$ which is also required for determining $k$ (constrained by $k>N(\varepsilon)$).
\section{Some Examples}
\label{ex}
This section presents examples of the PNF motion planning using MATLAB. We set the world's radius to $45$ units length.
Fig. \ref{fig:PNF's}.a depicts a stochastic scenario where the obstacles radii from the top right c.w. are $4,4$ and $2$, while the locations' STDs are $6,5 $, and $4$ respectively. The robot radius is $3$ and its location STD is $4$. $\Delta$  is chosen as $0.1$ while $k$ is chosen to be $5$ empirically (the larger $k$ is, the closer the PNF allows the robot to approach the obstacles).  
Fig. \ref{fig:PNF's}.b depicts a scenario where the obstacles have the same geometry, while the STDs are $30,5$, and $4$ respectively. $k$  is again chosen as $5$ , and $\Delta$ remains $0.1$.
In Fig.\ref{fig:PNF's}.c we use the same geometries and the same standard deviations as in Fig.\ref{fig:PNF's}.a, but $k$  is chosen as $2$  and $\Delta$ remains the same, (observe that the PNF seems farther to the obstacles). \\
As for a different selections of $\Delta$, in Fig.\ref{fig:PNF's}.d $\Delta=0.6$ and the path length is $77$ units, where in Fig.\ref{fig:PNF's}.e $\Delta=0.8$ which results in path length of $84$ units. \\
Finally, Fig.\ref{fig:PNF's}.f depicts a simulation with two slightly different initial configurations ($S_1$ and $S_2$) which results with bifurcation. Moreover, note that poorly chosen constant $k$ ($2$) results in undesirable local minima located at two points.  
In this case $N(\varepsilon)$ is large since the obstacles are close to each other, resulting in a small $\varepsilon$ (see Prop. \ref{prop_d}). Recall that $k>N(\varepsilon)$ and thus $k$ should set larger avoid bifurcation.

\begin{figure}[h]
\centering
\includegraphics[width=1.05\linewidth]{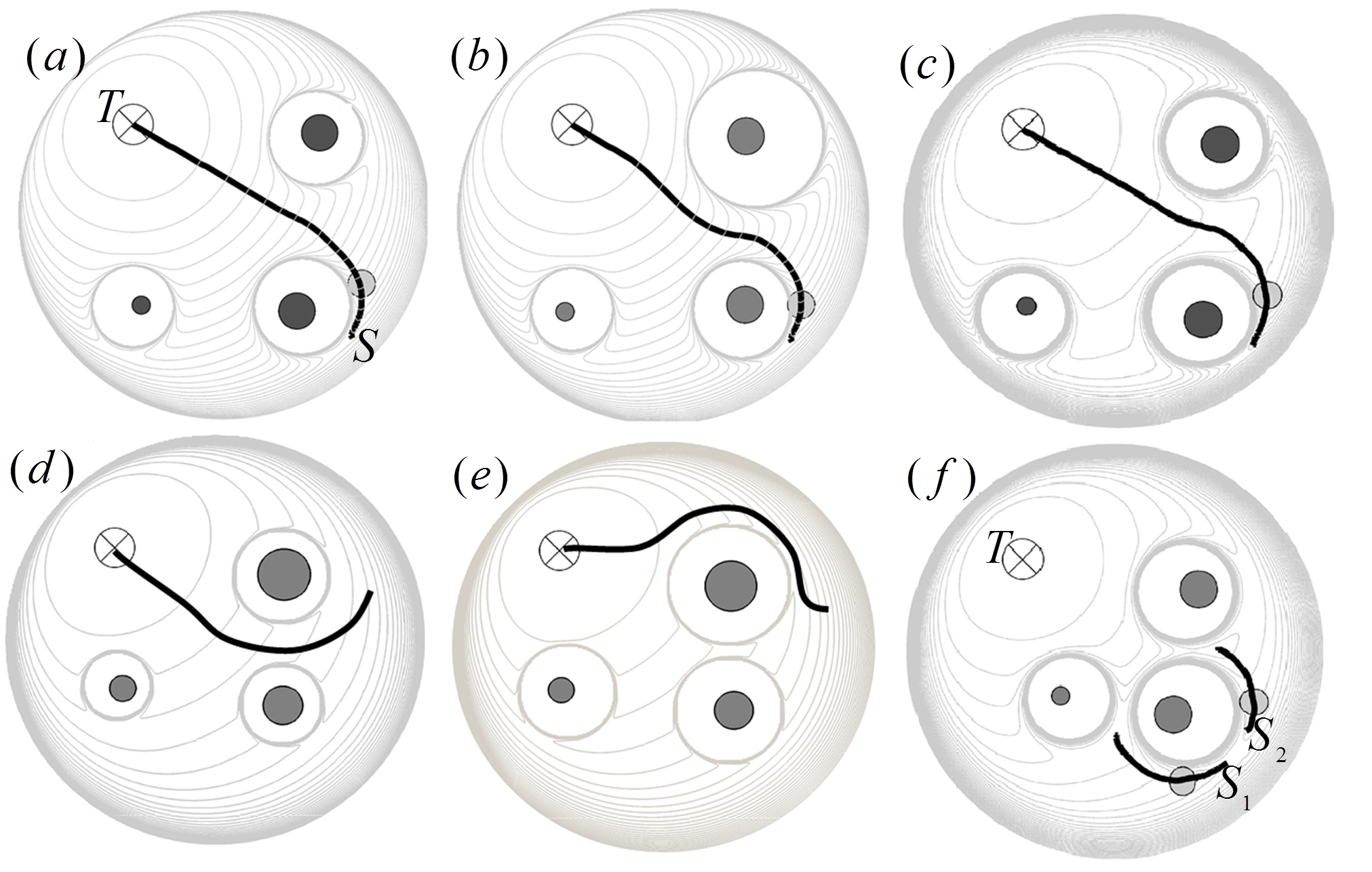}
\caption{Implementation of the PNF for sphere-world motion planning problem. The target is marked as $\otimes$. The bold line indicates the path from the initial point to the target (or to local minimum at $f$). The dark solid discs are the obstacles and the light solid disc the robot}
\label{fig:PNF's}
\end{figure}
Table 1 compares the performance of a PNF with a path planning generated by a traditional NF where we considered an inflated geometry of the obstacles with radii equivalent to $\Delta$ (e.g. for $\Delta=0.9$ the radius addition is $1.67$ STDs). Note that prior to constructing the NF, we performed a convolution of the robot's geometry with the geometry of the obstacles (as performed for the PNF). This was essential in order to compare the resulting paths from the two functions.
\begin{table}[h]
\begin{center}
\begin{tabular}{|c|c|c|c|c|}
\hline 
\small  {$\Delta$} &  {Method} & Path Length &  {STD} &  Failure [$\%$]     \\
 \hline\hline
\small \multirow{2}{*}{0.9}  & PNF & 46.74  & 28.38 & 1.53 \\

\small  & NF & 38.95 & 10.87  & 8.56  \\
 \hline
\small  \multirow{2}{*}{0.67} & PNF &  44.09 & 25.16  & 2.64\\
 
 \small & NF & 37.20 & 7.48 & 9.02  \\
\hline 
\end{tabular} 
\end{center}
\vspace{3mm}
{\small Table 1: Performance comparison of the PNF and NF with different $\Delta$s. Numbers are the average of $200$ different simulations with the same distributions and geometries. Failure refers to an obstacle-robot collision.}
\end{table}
\section{Summary}
We defined a probabilistic navigation function, such that following its gradient produces a path that decreases the probability for collision with the obstacles and converges to the target point.
In order to provide a "safe" motion path, we included an additional requirement for a maximal permitted probability for collision.\\  
We have introduced such a function $\varphi$, defined on $\C\subset \R^n$ and showed that $\varphi$ is indeed a probabilistic NF. \\
We proved that the  PNF converges for all stochastic scenarios. In order for the analysis to be as analytic as possible we assumed disc-shaped elements and Radial Gaussian distributions to model the uncertainties. That is, given a disc-shaped robot and disc-shaped obstacles with given uncertainties in their locations (in a disc-shaped world), we have shown how to construct $\varphi$ which will safely transverse to the target.  
Note that the discussion in this paper can be generalized to star-shaped worlds as well, in exactly the same manner as used in \cite{rimon1992exact}.  \\
We have demonstrated our algorithm on various scenarios, showing how the selection of $k$ affects the resulting paths. We also provided experimental results showing the effect of the extent of uncertainty on the path. Lastly, we compared the PNF to a simple NF showing that the resulting path from the PNF is safer (but naturally longer).\\
The PNF can be further extended to algorithmically include the robot's dynamics see for example \cite{howard2006trajectory}. In future work we intend to apply the PNF to the more general problem of stochastic-dynamic environment and to include generalized Gaussian distributions and geometries. The authors also wish to continue investigating a version where there is no assumption for pairwise obstacle distances- this is done by composing a second function (similar to that introduced in \cite{rimon1992exact}) that can handle the case of non-spherical unified obstacles.  
\bibliographystyle{plain} 
\bibliography{PNF_bib}
\appendix
\section*{Appendix} 
\small
Now, we shall prove some of the bounding $\varepsilon$'s we used in Section.\ref{sec is NF}.

\begin{lemma}
\label{lemma_max_nab_beta_i}
$\max\limits_{ \q} \{\| {\nabla {\beta _i}} \|\}\le \frac{e^{-\frac{1}{2}}}{\sqrt{2\pi} \sigma_i^2}$ 
and,\\ $\max\limits_{ \q} \{\| {\nabla^2 {\beta _i}} \|\}\le \sqrt{\frac{2}{\pi}}\frac{e^{-\frac{3}{2}}}{\sigma_i^3} $.
\end{lemma}
\begin{proof}
Throughout the paper  $\| \ \|$ denoted  the Euclidean norm.  Here we use $\| \ \|_p$ to indicate the general p-norm (e.g. $\| \ \|_2=\|\ \ \|$ ). \--\
Recall that $\beta_i$ is based on the convolution of the disc with a Gaussian. Thus as a consequence of Young's inequality \cite{fournier1977sharpness}, $\| \nabla \beta_i\|_2$ can be written as:
\begin{equation*}
\|\nabla \left( D(r,R_i)  *G(r,\sigma_i) \right)\|_2 = \|D(r,R_i) *\nabla G(r,\sigma_i)\|_2
\end{equation*}
Again using Young's inequality, this amounts to:
\begin{equation*}
\|D(\rr,R_i) *\nabla G(r,\sigma_i)\|_2\leq c_{2,1}\|D(\rr,R_i)\|_2 *\|\nabla G(\rr,\sigma_i)\|_1
\end{equation*}
where $c_{2,1}<1$. Since $D(\rr,R_i)$ is a disc with a unit height we have:

\begin{equation*}
\max\{\|\nabla \beta_i\|\}\leq \max\{\|\nabla G(\rr,\sigma_i)\|_1\} = \frac{e^{-\frac{1}{2}}}{\sqrt{2\pi} \sigma_i^2}
\end{equation*}
Using the same logic:

\begin{equation*}
\max\{\|\nabla^2 \beta_i\|\}\leq \max\{\|\nabla^2 G(\rr,\sigma_i)\|_1\} = \sqrt{\frac{2}{\pi}}\frac{e^{-\frac{3}{2}}}{\sigma_i^3}
\end{equation*}
 \begin{flushright}   \small$\blacksquare$ \normalfont \end{flushright}  
\end{proof}
\begin{lemma}
\label{max_bar_beta}
\begin{equation*}
\max\limits_{\q \in \B_i(\varepsilon)}\{ \bar{\beta}_i \} =\prod\limits_{j \in  \{\Ob-i\}}  p_{\Delta_j}-p_{tot}(\|\q_j-\q_i\| +R_{\varepsilon},R_j, \sigma_j) 
\end{equation*}  
\begin{equation*}
\min\limits_{\q \in \B_i(\varepsilon)}\{ \bar{\beta}_i \} =\prod\limits_{j \in  \{\Ob-i\}}  p_{\Delta_j}-p_{tot}(\|\q_j-\q_i\| -R_{\varepsilon},R_j, \sigma_j)
\end{equation*} 
\end{lemma}
\begin{proof}
Since $\bar{\beta_i}= \prod\limits_{j\in \{\Ob-i\}} \beta_j$  we have 
\begin{equation*}
\max\limits_{\q \in \B_i(\varepsilon)}\{ {\beta}_j \} = p_{\Delta_j}-p_{tot}(\|\q_j-\q_i\| +R_{\varepsilon},R_j, \sigma_j) \end{equation*}  
where $R_{\varepsilon}$ is a scalar that satisfies $p_{tot}(\|\q_i+R_{\varepsilon}\|,R_j,\sigma_j) =\varepsilon$. In the same way we obtain the second result.
\begin{flushright}   \small$\blacksquare$ \normalfont\end{flushright}  
\end{proof}

\begin{lemma}
\label{max_nabla_bar}
$\max\{\|\nabla \bar{\beta}_i\| \} \leq \frac{1}{\sqrt{2\pi e}}\sum\limits_{j\in \{\Ob-i\}}\frac{1}{\sigma_j^2} $
\end{lemma}

\begin{proof}
\begin{equation*}
\nabla \bar{\beta}_i=\sum_{j \in \{\Ob-i\}}{\nabla \beta_j\prod\limits_{k\ne i,j}\beta_k}.
\end{equation*}
The result follows since $\max\limits_q\{\beta_i\}=1$ and by Lemma \ref{lemma_max_nab_beta_i}.
 \begin{flushright}   \small$\blacksquare$ \normalfont\end{flushright}  
\end{proof}

\end{document}